\documentclass{article}

\usepackage{iclr2015,times}
\usepackage[english]{babel}
\usepackage[utf8]{inputenc}
\usepackage{natbib}
\usepackage{fancyvrb}
\usepackage{amsmath, amsthm, amssymb, latexsym}
\usepackage{algorithm}
\usepackage{algorithmic}
\usepackage{graphicx}
\usepackage{color}
\usepackage{hyperref}
\usepackage{url}

\allowdisplaybreaks

\title{Algorithmic Robustness for Learning via $(\epsilon, \gamma, \tau)$-Good Similarity Functions}
\author{
Maria-Irina Nicolae, Marc Sebban \& Amaury Habrard \\
Hubert Curien Laboratory\\
Jean Monnet University, Saint-Etienne, France \\
\texttt{\{Maria.Irina.Nicolae,Marc.Sebban,Amaury.Habrard\}@univ-st-etienne.fr} \\
\AND
Éric Gaussier \& Massih-Reza Amini \\
Laboratoire d'Informatique de Grenoble \\
Joseph Fourier University, Grenoble, France \\
\texttt{\{Eric.Gaussier,Massih-Reza.Amini\}@imag.fr}
}

\newtheorem{theorem}{Theorem}

\newtheorem{definition}{Definition}
\newtheorem{example}{Similarity function}

\iclrfinalcopy % Uncomment for camera-ready version

\begin{document}
\maketitle

% \keywords{semi-supervised learning, similarity function, algorithmic robustness}

\begin{abstract}
The notion of metric plays a key role in machine learning problems such as classification, clustering or ranking. 
However, it is worth noting that there is a severe lack of  theoretical guarantees that can be expected on the generalization capacity of the classifier associated to a given metric.
The theoretical framework of $(\epsilon, \gamma, \tau)$-good similarity functions \citep{conf/colt/BalcanBS08} has been one of the first attempts to draw a link between the properties of a similarity function and those of a linear classifier making use of it.
In this paper, we extend and complete this theory by providing a new generalization bound for the associated classifier based on the algorithmic robustness framework.
\end{abstract}

\section{Introduction}

Most of the machine learning algorithms make use of metrics for comparing objects and making decisions (e.g. SVMs, k-NN, k-means, etc.). However, it is worth noticing that the theoretical guarantees of these algorithms are always derived independently from the peculiarities of the metric they make use of. For example, in supervised learning, the generalization bounds on the classification error do not take into account the discriminative properties of the metrics. In this context, \cite{conf/colt/BalcanBS08} filled this gap by proposing the first framework that allows one to relate similarities with a classification algorithm. This general framework, that can be used with any bounded similarity function, provides generalization guarantees on a linear classifier learned from the similarity. Moreover, their algorithm, whose formulation is equivalent to a relaxed $L_1$-norm SVM~\citep{zhu20041}, does not enforce the positive definiteness constraint of the similarity.
In this paper, we show that using Balcan et al's setting and the algorithmic robustness framework \citep{DBLP:journals/ml/XuM12}, we can derive  generalization guarantees which  consider other properties of the similarity. This leads to new consistency bounds for different kinds of similarity functions.

\section{Notations and Related Work}\label{sec:notations}
%We denote vectors by lower-case bold symbols ($\mathbf{x}$) and matrices by upper-case bold symbols ($\mathbf{A}$).
Let us assume we are given access to labeled examples $\mathbf{z}=(\mathbf{x}, l(\mathbf{x}))$ drawn from some unknown distribution $P$ over $\mathcal{X}\times\mathcal{Y}$, where  $\mathcal{X}\subseteq\mathbb{R}^d$ and  $\mathcal{Y}=\{-1,1\}$ are respectively the instance and the output spaces.
A pairwise similarity function $K_{\mathbf{A}}$ over $\mathcal{X}$, possibly parameterized by a matrix $\mathbf{A}\in\mathbb{R}^{d\times d}$, is defined as $K_{\mathbf{A}}:\mathcal{X}\times \mathcal{X}\rightarrow [-1,1]$, and the hinge loss as $[c]_+=\max(0, 1-c)$.
We denote the $L_1$ norm by $||\cdot||_1$, the $L_2$ norm by $||\cdot||_2$ and the Frobenius norm by $||\cdot||_{\mathcal{F}}$. We assume that $||\mathbf{x}||_2 \leq 1$.

\cite{conf/colt/BalcanBS08} introduced a theory for learning with so called $(\epsilon, \gamma,\tau)$-good similarity functions. %This was the first stone to establish generalization guarantees for a linear classifier that would be learned by making use of such similarities. 
Their generalization guarantees are based on the following definition.

\begin{definition}\citep{conf/colt/BalcanBS08}\label{def:hinge}
$K_{\mathbf{A}}$ is a $(\epsilon, \gamma,\tau)$-good similarity function in hinge loss for a learning problem P if there exists a random indicator function $R(\mathbf{x})$ defining a probabilistic set of "reasonable points" such that the following conditions hold:
\begin{enumerate}
	\item $\mathbb{E}_{(\mathbf{x}, l(\mathbf{x}))\sim P}\left[\left[1-l(x)g(\mathbf{x})/\gamma\right]_+\right]\leq \epsilon,$\\
	where $g(\mathbf{x}) = \mathbb{E}_{(\mathbf{x}',l(\mathbf{x}'),R(\mathbf{x}'))}\left[l(\mathbf{x}')K_{\mathbf{A}}(\mathbf{x}, \mathbf{x}')|R(\mathbf{x}')\right]$.
	\item $\Pr_{\mathbf{x}'}(R(\mathbf{x}')) \geq \tau$.
\end{enumerate}
\end{definition}

This definition imposes a constraint on the mass of reasonable points one must consider (greater than $\tau$). It also expresses the tolerated margin violations in an averaged way: a $(1-\epsilon)$ proportion of examples $\mathbf{x}$ are on average $2\gamma$ more similar to random reasonable examples $\mathbf{x}'$ of their own label than to random reasonable examples $\mathbf{x}'$ of the other label. %This allows for more flexibility than pair- or triplet-based constraints. Notice that no constraint is imposed on the form of the similarity function. 
Definition~\ref{def:hinge} can then be used to learn well:

\begin{theorem}\citep{conf/colt/BalcanBS08}\label{theo:hinge}
Let $K_{\mathbf{A}}$ be an $(\epsilon, \gamma, \tau)$-good similarity function in hinge loss for a learning problem P.
For any $\epsilon_1>0$ and $<\delta<\gamma\epsilon_1/4$ let $\mathcal{S}=\{\mathbf{x}'_1, \mathbf{x}'_2,\ldots,\mathbf{x}'_{d_u}\}$ be a sample of $d_u=\frac{2}{\tau}\left ( log(2/\delta)+16\frac{log(2/\delta)}{(\epsilon_1\gamma)^2} \right )$ landmarks drawn from P.
Consider the mapping $\phi^\mathcal{S}:\mathcal{X}\rightarrow \mathbb{R}^{d_u}$ defined as follows: $\phi^\mathcal{S}_i(\mathbf{x})=K_{\mathbf{A}}(\mathbf{x}, \mathbf{x}'_i), i\in \{1,\ldots, d_u\}$.
With probability $1-\delta$ over the random sample $\mathcal{S}$, the induced distribution $\phi^\mathcal{S}(P)$ in $\mathbb{R}^{d_u}$, has a separator achieving hinge loss at most $\epsilon+\epsilon_1$ at margin $\gamma$.
\end{theorem}

In other words, if $K_{\mathbf{A}}$ is $(\epsilon, \gamma, \tau)$-good according to Definition~\ref{def:hinge} and enough points are available, there exists a linear separator $\boldsymbol{\alpha} \in \mathbb{R}^{d_u}$ with error arbitrarily close to $\epsilon$ in the space $\phi^\mathcal{S}$.
This separator can be learned from $d_l$ labeled examples by solving the following optimization problem:
%The procedure for finding the separator involves two steps: first using $d_u$ {\bf potentially unlabeled examples} as landmarks  to construct the feature space, then using a new labeled set of size $d_l$ to estimate $\boldsymbol{\alpha}\in\mathbb{R}^{d_u}$.
%This is done by solving the following optimization problem:

\begin{eqnarray}
\min\frac{1}{d_l}\sum_{i=1}^{d_l} \ell(\mathbf{A},\boldsymbol{\alpha},\mathbf{z}_i) \hspace{1cm} \mbox{ s.t. } \sum_{j=1}^{d_u} |\alpha_j|\leq 1/\gamma \label{eq:linear-problem}
\end{eqnarray}
where
$\ell(\mathbf{A},\boldsymbol{\alpha},\mathbf{z}_i=(\mathbf{x}_i,l(\mathbf{x}_i))) = \left [1-\sum_{j=1}^{d_u}\boldsymbol{\alpha}_jl(\mathbf{x}_i)K_{\mathbf{A}}(\mathbf{x}_i,\mathbf{x}_j) \right ]_+$ is the instantaneous loss estimated at point $(\mathbf{x}_i,l(\mathbf{x}_i))$. Therefore, this optimization problem reduces to minimizing the empirical loss ${\cal \hat{R}}^{\ell}=\frac{1}{d_l}\sum_{i=1}^{d_l}\ell(\mathbf{A},\boldsymbol{\alpha},\mathbf{z}_i)$ over the training set ${\cal S}$.
%\begin{eqnarray}
%\min_{\boldsymbol{\alpha}} \left\{ \sum_{i=1}^{d_l} \left[1-\sum_{j=1}^{d_u} \alpha_j l(\mathbf{x}_i) K_{\mathbf{A}}(\mathbf{x}_i, \mathbf{x}_j)\right]_+  : \sum_{j=1}^{d_u} |\alpha_j|\leq 1/\gamma \right\}. \label{eq:linear-problem}
%\end{eqnarray}
Note that this problem can be solved efficiently by linear programming.
Also, as the problem is $L_1$-constrained, tuning the value of $\gamma$ will produce a sparse solution.\\

\section{Consistency Guarantees}
In this section, we provide a new generalization bound for the classifier learned in Problem (\ref{eq:linear-problem}) based on the recent algorithmic robustness framework proposed by \cite{DBLP:journals/ml/XuM12}. To begin with, let us recall the notion of robustness of an algorithm ${\cal A}$.

\begin{definition}[Algorithmic Robustness \citep{DBLP:journals/ml/XuM12}]
Algorithm $\mathcal{A}$ is $(M,\epsilon(\cdot))$-robust, for $M\in\mathbb{N}$ and $\epsilon(\cdot):\mathcal{Z}^{d_l}\to\mathbb{R}$, if $\mathcal{Z}$ can be partitioned into $M$ disjoint sets, denoted by $\{C_i\}_{i=1}^M$, such that the following holds for all $\mathcal{S}\in\mathcal{Z}^{d_l}$:
\begin{align*}
\forall \mathbf{z}=(\mathbf{x},l(\mathbf{x}))\in \mathcal{S}, \forall \mathbf{z}'=(\mathbf{x}',l(\mathbf{x}'))\in\mathcal{Z},\forall i\in[M]: \\
\text{if } \mathbf{z},\mathbf{z}'\in C_i, \text{then } |\ell(\mathbf{A},\boldsymbol{\alpha},\mathbf{z})-\ell(\mathbf{A},\boldsymbol{\alpha},\mathbf{z}')|\leq \epsilon(\mathcal{S}).
\end{align*}
\end{definition}

%This notion is a desired property of a learning algorithm, as it implies a lack of sensitivity to small perturbations in the training data in the same sense as robust optimization.
Roughly speaking, %an algorithm is robust if for any example $\mathbf{z}'$ falling in the same subset as a training example $\mathbf{z}$, the gap between the losses associated with $\mathbf{z}$ and $\mathbf{z}'$ is bounded.
%In other words, 
robustness characterizes the capability of an algorithm to perform similarly on close train and test instances. The closeness of the instances is based on a partitionning of $\mathcal{Z}$: two examples are close if they belong to the same region. 
In general, the partition is based on the notion of covering number \citep{Kolmogorov-Tikhomirov61} allowing one to cover $\mathcal{Z}$ by regions where the distance/norm between two elements in the same region are no more than a fixed quantity $\rho$ (see \cite{DBLP:journals/ml/XuM12} for details about how the convering is built).
%. The covering is built as follows: first we consider a $\rho$-cover over the instance $\mathcal{X}$, then we partition $\mathcal{Z}$ by considering one $\rho$-cover over $\mathcal{X}$ for the positive instances and another $\rho$-cover over $\mathcal{X}$ for the negative instances ensuring that two examples in the same region belong to the same class and the distance between them is no more than $\rho$ (\cite{DBLP:conf/colt/XuM10,DBLP:journals/ml/XuM12} for details). Note that with this construction the \text{$\rho$}-covers verify the cluster assumption used in semi-supervised learning \citep{ChaSchZie06}.
Now we can state the first theoretical contribution of this paper. 

\begin{theorem} \label{theo1}
Given a partition of ${\cal Z}$ into $M$ subsets $\{C_i\}$ such that $\mathbf{z}=(\mathbf{x},l(\mathbf{x}))$ and $ \mathbf{z}'=(\mathbf{x}',l(\mathbf{x}')) \in C_i$ and $l(\mathbf{x})=l(\mathbf{x}')$, and provided that $K_{\mathbf{A}}(\mathbf{x},\mathbf{x}')$ is $l$-lipschitz w.r.t.\ its first argument, the optimization problem~\eqref{eq:linear-problem} is $(M,\epsilon(\mathcal{S}))$-robust with $\epsilon(\mathcal{S})=\frac{1}{\gamma}l\rho$, where $\rho=\sup_{\mathbf{x},\mathbf{x}' \in C_i}||\mathbf{x}-\mathbf{x}'||$.
\end{theorem}

\begin{proof}[Proof]
%\begin{small}
\begin{align}
\left| \ell(\mathbf{A},\boldsymbol{\alpha},\mathbf{z})-\ell(\mathbf{A},\boldsymbol{\alpha},\mathbf{z}') \right| \leq & \left|\sum_{j=1}^{d_u}\alpha_jl(\mathbf{x}')K_{\mathbf{A}}(\mathbf{x}',\mathbf{x}_j)-\sum_{j=1}^{d_u}\alpha_jl(\mathbf{x})K_{\mathbf{A}}(\mathbf{x},\mathbf{x}_j)\right| \label{eq:1}\\
%= & \left|\sum_{j=1}^{d_u}\alpha_j(K_{\mathbf{A}}(\mathbf{x}',\mathbf{x}_j)-K_{\mathbf{A}}(\mathbf{x},\mathbf{x}_j))\right| \nonumber \\
%\leq & \sum_{j=1}^{d_u} ||\alpha_j||_2\cdot||K_{\mathbf{A}}(\mathbf{x}',\mathbf{x}_j)-K_{\mathbf{A}}(\mathbf{x},\mathbf{x}_j)||_2\label{eq:2}\\
%\leq & \sum_{j=1}^{d_u} |\alpha_j|\cdot||K_{\mathbf{A}}(\mathbf{x}',\mathbf{x}_j)-K_{\mathbf{A}}(\mathbf{x},\mathbf{x}_j)||_1\label{eq:3}\\
\leq & \sum_{j=1}^{d_u} |\alpha_j|\cdot\left|K_{\mathbf{A}}(\mathbf{x}',\mathbf{x}_j)-K_{\mathbf{A}}(\mathbf{x},\mathbf{x}_j)\right| \label{eq:3} \\
\leq & \sum_{j=1}^{d_u} |\alpha_j|\cdot l||\mathbf{x}-\mathbf{x}'||
\leq \frac{1}{\gamma}l\rho \label{eq:4}
%\leq & ||\boldsymbol{\alpha}||_1.||K_{\mathbf{A}}(\mathbf{x}',\mathbf{x}_j)-K_{\mathbf{A}}(\mathbf{x},\mathbf{x}_j))||_1\label{eq:3}\\
%\leq & \frac{1}{\gamma}\sum_{j=1}^{d_u}|K_{\mathbf{A}}(\mathbf{x},\mathbf{x}_j)-K_{\mathbf{A}}(\mathbf{x}',\mathbf{x}_j)|\label{eq:4}\\
%\leq & \frac{1}{\gamma}\frac{1}{d_u}d_ul||\mathbf{x}-\mathbf{x}'|| \label{eq:5}\\
\end{align}
%\end{small}

Setting $\rho=\sup_{\mathbf{x},\mathbf{x}' \in C_i}||\mathbf{x}-\mathbf{x}'||_1$, we get the Theorem.
We get Inequality~\eqref{eq:1} from the 1-lipschitzness of the hinge loss; 
%Inequalities~\ref{eq:2} and~\ref{eq:3} come from the Cauchy-Schwarz inequality and some classical norm properties; 
Inequality \eqref{eq:3} comes from the classical triangle inequality;
The first inequality on line~\eqref{eq:4} is due to the $l$-lipschitzness of $K_{\mathbf{A}}(\mathbf{x},\mathbf{x}_j)$ and the result follows from the constraint of Problem~\eqref{eq:linear-problem}. 
\end{proof}

%The proof of Theorem \ref{theo1} is given in Appendix.
We now give a PAC generalization bound on the true loss making use of the previous robustness result. Let ${\cal R}^{\ell}=\mathbb{E}_{\mathbf{z} \sim {\cal Z}}\ell(\mathbf{A},\boldsymbol{\alpha},\mathbf{z})$ be the true loss w.r.t.\ the unknown distribution ${\cal Z}$ and ${\cal \hat{R}}^{\ell}=\frac{1}{d_l}\sum_{i=1}^{d_l}\ell(\mathbf{A},\boldsymbol{\alpha},\mathbf{z}_i)$ be the empirical loss over the training set ${\cal S}$. 
%Based on the results of \cite{DBLP:conf/colt/XuM10,DBLP:journals/ml/XuM12}, the proof requires the use of the following concentration inequality over multinomial random variables allowing one to capture statistical information coming from the different regions of the partition of $\mathcal{Z}$.

%\begin{proposition}\citep{van1996weak}\label{propo1} \\
%Let $(|N_1|,\dots,|N_M|)$ an i.i.d. multinomial random variable with
%parameters $d_l=\sum_{i=1}^M|N_i|$ and $(p(C_1),\dots,p(C_M))$. 
%By the Bretagnolle-Huber-Carol inequality we have:
%$
%\Pr\left\{\sum_{i=1}^M \left|\frac{|N_i|}{d_l}-p(C_i)\right| \geq
%  \lambda \right\}\leq 2^M \exp\left(\frac{-d_l\lambda^2}{2}\right)
%$, 
%hence with probability at least $1-\delta$,
%\begin{equation}
%\sum_{i=1}^M \left|\frac{N_i}{d_l}-p(C_i)\right|\leq \sqrt{\frac{2M\ln
%  2 + 2 \ln(1/\delta)}{d_l}}.
%\end{equation}
%\end{proposition}

%We are now able to present our generalization bound thanks to the following theorem.

\begin{theorem}\label{theo2}
Considering that problem~\eqref{eq:linear-problem} is $(M,\epsilon(\mathcal{S}))$-robust, and that $K_{\mathbf{A}}$ is $l$-lipschitz w.r.t.\ to its first argument, for any $\delta >0$ with probability at least $1-\delta$, we have:
$$|{\cal R}^{\ell}-{\cal \hat{R}}^{\ell}| \leq \frac{1}{\gamma}l\rho +B\sqrt{\frac{2M\ln2 + 2\ln(1/\delta)}{d_l}}, $$ \label{theorem2}
where $B=1+\frac{1}{\gamma}$ is an upper bound of the loss $\ell$.
\end{theorem}

The proof of Theorem~\ref{theo2} follows the one described in~\cite{DBLP:journals/ml/XuM12} and makes use of a concentration inequality over multinomial random variables \citep{van1996weak}. Note that in robustness bounds, the cover radius $\rho$ can be made arbitrarily small at the expense of larger values of $M$. As $M$ appears in the second term, which decreases to 0 when $d_l$ tends to infinity, this bound provides a standard $O(1/\sqrt{d_l})$ asymptotic convergence. 

%\subsection{Robustness Analysis for Different Similarity Functions}\label{sec:ex}

The previous theorem strongly depends on the $l$-lipschitzness of the similarity function.
In the following, we focus on some particular similarities that can be used in this setting: $K^1_{\mathbf{A}}$, a similarity derived from the Mahalanobis distance, $K^2_{\mathbf{A}}$ a bilinear similarity and $K^3_{\mathbf{A}}$ an exponential similarity.
We provide the proof of the $l$-lipschitzness for $K^1_{\mathbf{A}}$. The two others follow the same ideas. % The proofs for the following functions are detailed in the Appendix.
%To get a new consistency result w.r.t.\ these similarity functions, we typically have to prove their $l$-lipschitzness.
%The first one, denoted by $K^4_{\mathbf{A}}(\mathbf{x},\mathbf{x}')=\exp \left (\frac{\mathbf{x}^T\mathbf{A}\mathbf{x}'-1}{2\sigma} \right )$, is based on the cosine similarity.
%These two similarities have the main advantage to be well suited to a $k$-NN classifier. Indeed, they approximate the indicator function $\mathbf{1}(\mathbf{x}_i \in \knn(\mathbf{x}))$ used in a $k$-NN by providing a much higher similarity to the points that are close to $\mathbf{x}$.
%It is worth noting that in both cases, parameter $\sigma$ determines in a way this close neighborhood (the smaller $\sigma$, the larger the similarity to the first neighbors) and tends to play the same role as that of $k$.

\begin{example}\label{ex2}
We define $K^1_{\mathbf{A}}(\mathbf{x},\mathbf{x}') = 1 - (\mathbf{x}-\mathbf{x}')^T\mathbf{A}(\mathbf{x}-\mathbf{x}')$, a similarity derived from the Mahalanobis distance.
$K^1_{\mathbf{A}}(\mathbf{x},\mathbf{x}')$ is $4||\mathbf{A}||_2$-lipschitz w.r.t.\ its first argument.
\end{example}

\begin{proof}[Proof]
\begin{scriptsize}
\begin{align}
\left| K^2_{\mathbf{A}}(\mathbf{x},\mathbf{x}'') - K^2_{\mathbf{A}}(\mathbf{x}',\mathbf{x}'') \right| = & \left | 1 - \left ( (\mathbf{x}-\mathbf{x}'')^T\mathbf{A}(\mathbf{x}-\mathbf{x}'') \right ) - 1 + \left ( (\mathbf{x}'-\mathbf{x}'')^T\mathbf{A}(\mathbf{x}'-\mathbf{x}'') \right ) \right | \nonumber \\
= & \left|(\mathbf{x}'-\mathbf{x}'')^T\mathbf{A}(\mathbf{x}'-\mathbf{x}'') - (\mathbf{x}'-\mathbf{x}'')^T\mathbf{A}(\mathbf{x}-\mathbf{x}'') \right. + \left. (\mathbf{x}'-\mathbf{x}'')^T\mathbf{A}(\mathbf{x}-\mathbf{x}'') - (\mathbf{x}-\mathbf{x}'')^T\mathbf{A}(\mathbf{x}-\mathbf{x}'') \right|\nonumber \\
= &  \left |(\mathbf{x}'-\mathbf{x}'')^T\mathbf{A}(\mathbf{x}'-\mathbf{x}) + (\mathbf{x}'-\mathbf{x})^T\mathbf{A}(\mathbf{x}-\mathbf{x}'') \right |\nonumber \\
\leq & \left |(\mathbf{x}'-\mathbf{x}'')^T\mathbf{A}(\mathbf{x}'-\mathbf{x}) \right | + \left |(\mathbf{x}'-\mathbf{x})^T\mathbf{A}(\mathbf{x}-\mathbf{x}'') \right |  \nonumber \\
\leq & ||\mathbf{x}'-\mathbf{x}''||_2 \cdot ||\mathbf{A}||_2\cdot||\mathbf{x}'-\mathbf{x}||_2+||\mathbf{x}'-\mathbf{x}||_2 \cdot ||\mathbf{A}||_2\cdot||\mathbf{x}-\mathbf{x}''||_2 \label{eq:13-4}\\
%& \leq & \left (||\mathbf{x}'-\mathbf{x}''||_2 \cdot (||\mathbf{A}\mathbf{x}'||_2+||\mathbf{A}\mathbf{x}||_2)+||\mathbf{x}'-\mathbf{x}||_2 \cdot (||\mathbf{A}\mathbf{x}||_2+||\mathbf{A}\mathbf{x}''||_2) \right )\label{eq:14-4}\\
\leq & ||\mathbf{x}'-\mathbf{x}''||_2 \cdot ||\mathbf{A}||_2\cdot (||\mathbf{x}'||_2+||\mathbf{x}||_2) + ||\mathbf{x}'-\mathbf{x}||_2 \cdot ||\mathbf{A}||_2\cdot (||\mathbf{x}||_2+||\mathbf{x}''||_2) \nonumber\\
\leq & 4  \cdot ||\mathbf{A}||_2  \cdot  ||\mathbf{x}-\mathbf{x}'||. \label{eq:16-4}
\end{align}
\end{scriptsize}
Inequality~\eqref{eq:13-4} comes from the Cauchy-Schwarz inequality and some classical norm properties; Inequality~\eqref{eq:16-4} comes from the assumption that $||\mathbf{x}||_2 \leq 1$.
\end{proof}

\begin{example}\label{ex1}
Let $K^2_{\mathbf{A}}$ be the bilinear form $K^2_{\mathbf{A}}(\mathbf{x},\mathbf{x}')=\mathbf{x}^T\mathbf{A}\mathbf{x}'$.
$K^2_{\mathbf{A}}(\mathbf{x},\mathbf{x}')$ is $||\mathbf{A}||_2 $-lipschitz w.r.t.\ its first argument.
\end{example}

\begin{example}\label{ex3}
Let $K^3_{\mathbf{A}}(\mathbf{x},\mathbf{x}')=\exp \left(-\frac{{(\mathbf{x}-\mathbf{x}')^T\mathbf{A}(\mathbf{x}-\mathbf{x}')}}{2\sigma^2} \right)$.
$K^3_{\mathbf{A}}(\mathbf{x},\mathbf{x}')$ is $l$-lipschitz w.r.t.\ its first argument with $l=\frac{2||\mathbf{A}||_2 }{\sigma^2} \left (\exp\left(\frac{1}{2\sigma^2}\right)-\exp\left(\frac{-1}{2\sigma^2}\right) \right )$.
\end{example}

%Notice that $K^3_{\mathbf{A}}$ is based on the Mahalanobis distance $d_{\mathbf{A}}(\mathbf{x},\mathbf{x}')=\sqrt{(\mathbf{x}-\mathbf{x}')^T\mathbf{A}(\mathbf{x}-\mathbf{x}')}$.
%Plugging $l=\frac{2}{\sigma^2} \left (\exp\left(\frac{1}{2\sigma^2}\right)-\exp\left(\frac{-1}{2\sigma^2}\right) \right )$ in Theorem~\ref{theorem2}, we get a consistency result for problem~\eqref{eq:obj} using $K^3_{\mathbf{A}}(\mathbf{x},\mathbf{x}')$.

% \begin{example} \label{ex4}
% Let $K^4_{\mathbf{A}}(\mathbf{x},\mathbf{x}')=\exp \left (\frac{\mathbf{x}^T\mathbf{A}\mathbf{x}'-1}{2\sigma} \right )$.
% $K^4_{\mathbf{A}}(\mathbf{x},\mathbf{x}')$ is $l$-lipschitz w.r.t.\ its first argument with $l=\frac{1}{2\sigma} \left ( 1-\exp\left(\frac{-1}{2\sigma}\right) \right )$.
% \end{example}

% Note that setting $\mathbf{A}$ to the identity matrix leads to the standard Gaussian kernel $K^2_I(\mathbf{x},\mathbf{x}')=\exp \left (-\frac{||\mathbf{x}-\mathbf{x}'||^2}{2\sigma^2} \right )$.
% $K^4_{\mathbf{A}}$ is bounded by the interval $[0,1]$ because we assume that $||\mathbf{x}||_2 \leq 1$, $|\mathbf{A}_{kk}| \leq 1$ and $\mathbf{A}$ is diagonal.
% Plugging $l=\frac{1}{2\sigma} \left ( 1-\exp\left(\frac{-1}{2\sigma}\right) \right )$ in Theorem~\ref{theorem2} yields a consistency result for problem~\eqref{eq:obj} using $K^4_{\mathbf{A}}(\mathbf{x},\mathbf{x}')$.

As both $K^1_{\mathbf{A}}$ and $K^2_{\mathbf{A}}$ are linear w.r.t.\ their arguments, they have the main advantage to keep problem~\eqref{eq:linear-problem} convex.
$K^3_{\mathbf{A}}$ is also based on the Mahalanobis distance, but this time it is a non linear function, ressembling more a gaussian kernel.
Plugging $l=4||\mathbf{A}||_2$ (resp. $l=||\mathbf{A}||_2$ and $l=\frac{2||\mathbf{A}||_2}{\sigma^2} \left (\exp\left(\frac{1}{2\sigma^2}\right)-\exp\left(\frac{-1}{2\sigma^2}\right) \right )$) in Theorem~\ref{theorem2}, we obtain consistency results for problem~\eqref{eq:linear-problem} using $K^1_{\mathbf{A}}(\mathbf{x},\mathbf{x}')$ (resp. $K^2_{\mathbf{A}}(\mathbf{x},\mathbf{x}')$ and $K^3_{\mathbf{A}}(\mathbf{x},\mathbf{x}')$).
As the gap between empirical and true loss presented in Theorem~\ref{theo2} is proportional with $l$ for the $l$-lipschitzness of each similarity function, we would like to keep this parameter as small as possible.
We notice that the generalization bound is tighter for $K^1_{\mathbf{A}}$ than for $K^2_{\mathbf{A}}$.
The bound for $K^3_{\mathbf{A}}$ depends on the additional parameter $\sigma$, that adjusts the influence of the similarity value w.r.t.\ the distance to the landmarks.
The value of $l$ goes to 0 as $\sigma$ augments, so larger values of $\sigma$ are preferable in order to obtain a tight bound for the generalization error. However, note that when $\sigma$ is large, the exponential behaves almost linearly, i.e. the projection loses its non-linear power.

\section{Conclusion}\label{sec:conclusion}
In this paper, we extended the theoretical analysis of the $(\epsilon, \gamma, \tau)$-good similarity framework. Using the algorithmic robustness setting, we derived new generalization bounds for different similarity functions. It turns out that the smaller the lipschitz constant of those similarity functions, the tighter the consistency bounds. This opens the door to new lines of research in {\it metric learning}~\citep{bellet2013survey,BelletHS2015} aiming at maximizing the $(\epsilon, \gamma, \tau)$-goodness of similarity functions s.t. $||\mathbf{A}||_2$ is as small as possible (see pioneer works like \cite{Bellet2012a,BelletHS11}).

\subsubsection*{Acknowledgements} 
Funding for this project was provided by a grant from Région Rhône-Alpes.

\bibliographystyle{iclr2015}
\begin{small}
\bibliography{references}
\end{small}

\end{document}